\documentclass{article}

\usepackage[final]{neurips_2022}

\usepackage[utf8]{inputenc} % allow utf-8 input
\usepackage[T1]{fontenc}    % use 8-bit T1 fonts
\usepackage{hyperref}       % hyperlinks
\usepackage{url}            % simple URL typesetting
\usepackage{booktabs}       % professional-quality tables
\usepackage{amsfonts}       % blackboard math symbols
\usepackage{nicefrac}       % compact symbols for 1/2, etc.
\usepackage{microtype}      % microtypography
\usepackage{xcolor}         % colors
\usepackage[ruled]{algorithm2e}
\usepackage{algorithmic}
\usepackage{bbm}
\usepackage[sort, numbers]{natbib}
\usepackage{amsthm}
\usepackage{graphicx}

\newcommand{\R}{{\mathbb{R}}}

\newtheorem{theorem}{Theorem}

\title{Renormalized Sparse Neural Network Pruning}

\author{%
  Michael G. Rawson \\
  Pacific Northwest National Lab and \\
  University of Maryland at College Park\\
  \texttt{michael.rawson@pnnl.gov} \\
}

\begin{document}

\maketitle

\begin{abstract}
Large neural networks are heavily over-parameterized. This is done because it improves training to optimality. However once the network is trained, this means many parameters can be zeroed, or pruned, leaving an equivalent sparse neural network. We propose renormalizing sparse neural networks in order to improve accuracy. We prove that our method's error converges to zero as network parameters cluster or concentrate. We prove that without renormalizing, the error does not converge to zero in general. We experiment with our method on real world datasets MNIST, Fashion MNIST, and CIFAR-10 and confirm a large improvement in accuracy with renormalization versus standard pruning.
\end{abstract}

\section{Introduction}
Sparse neural networks are being studied more as neural networks get deployed to edge devices and internet-of-things devices \cite{mao2017exploring}. Sparse neural networks have many parameters set to zero to decrease both computation and memory requirements. The simplest and most common pruning method is to train a neural network and then set the smallest weights, in absolute value, to zero. This `one-shot' method can sparsify over 90\% of the network's parameters with less than 1\% accuracy loss \cite{Liu_2015_CVPR, mao2017exploring}. We have reproduced some of these results. In Section 2, we describe our method and give the relevant theorems. In Section 3, we give experimental results and discuss the experimental setup. In Section 4, we conclude. 

\section{Method and Theoretic Results}
We consider the simplest pruning method that sets a subset of the network parameters to zero and does not do any further training or optimization, also known as `one-shot'. Also we consider standard training and no special sparsity inducing training such as \cite{louizos2017learning, zimmer2022compression, grigas2019stochastic, miao2021learning}. We propose Renormalized Pruning which is a rescaling of the sparsified parameters, see Algo. \ref{algo1}.

We prove that the Renormalized Pruning approximation error is bounded and the error is consistent in terms of concentrations of neural network parameters. Renormalized Pruning is accurate up to two terms of the variance or range of neural network parameters, $a_i$, and features (or embeddings), $\phi_i$, times the amount of pruning done, $M$. The error converges to zero at a linear rate. Then we show that standard pruning does not converge in general, in this scenario. Note that the $M$-sparse neural network, $\Phi_M$, need not remove the smallest $a_i$. This is very relevant in the very high sparsity regime, >90\% sparsity. At a neural network's initial randomization, neural networks produce random features for a given input and then linearly combine them in the final layer to make a prediction. We can write this approximation of function $f$ as $f(x) \approx \sum_i a_i \phi_i(x)$ where $f$ is the function to approximate or predict, $a_i\in\R$, and function $\phi_i$ is randomly sampled from a distribution, for each $i$. Random features, also known as random embeddings, are good approximators when $a_i$ is chosen optimally,
via $\{a_i\}_i = \arg\min_{\{a_i\}_i} \| f - \sum_i a_i \phi_i \| + \lambda \mathcal{R} (\{ a_i \}_i)$ with norm $\|\cdot\|$ and regularizer $\mathcal{R}$, see \cite{hashemi2021generalization, mei2022generalization}.

\begin{center}
    \begin{minipage}{.7\linewidth}
\begin{algorithm}[H]
	\caption{Renormalized Pruning}
	\label{algo1}
    \textbf{Input:} \\
\quad    $ \Phi $ : Neural Network \\
\quad    $prune()$ : pruning function \\
\quad    $\epsilon \in \R$ : positive pruning threshold \\
\quad    TrainNeuralNetwork() : Training method \\
\quad    GetParameters() : Returns network parameters \\
\quad    LoadParameters() : Loads network parameters \\
    \textbf{Output:} \\
\quad    $ \Phi $ : Sparse Neural Network \\
    \textbf{Begin:} \\
\quad    TrainNeuralNetwork($\Phi$) \\
\quad    $v$ = GetParameters($\Phi$) \\
\quad    $w = prune(v, \epsilon)$ \\
\quad    $w = \frac{\|v\|_0}{\|w\|_0} w$ \\
\quad    LoadParameters($\Phi,\ w$) \\
\end{algorithm}
    \end{minipage}
\end{center}

Let neural network $\Phi$ approximate function $f$ and write as $\Phi(x) = \sum_i a_i \phi_i(x)$. We call $\{a_i\}_i$ the parameters of the neural network and assume $N$ of the $\{a_i\}$ are nonzero. Let $\Phi_M$ be the sparse neural network after setting $M$ parameters to 0. Define $\mu_\phi$ to be the expectation of random variable $\phi$ and $\Delta_{\phi} = \max_\phi \|\phi-\mu_\phi\|$. 

\begin{theorem} \label{thm:1}
Let the renormalized network $\widehat \Phi_M := \frac{N}{N-M} \Phi_M$. 
Assume all $\phi_i$ are in the $\delta$ radius 2-Ball, $\phi_i \in B_{\delta}$, that is $\|\phi_i\|_{2} \le \delta$.
Let $a_i \in [\alpha, \beta]$ with $0<\alpha<\beta$ for each $i$ by absorbing the sign into $\phi_i$.
Let $\xi := \beta-\alpha$ and let $P$ be the set of indices pruned in $\Phi_M$. Then 

$$  \| \Phi - \widehat\Phi_M \|_2
\le 2 \beta \Delta_\phi M
+ \xi (\beta+\alpha)
\delta M/\alpha. $$

This goes to 0 as random variables $\phi_i$ and $a_i$ concentrate,

$$  \| \Phi - \widehat\Phi_M \|_2
 \le
 O(\Delta_\phi+\xi)
$$
as $\Delta_\phi + \xi \rightarrow 0$.

\end{theorem}

Proof in appendix. 

However, with this setting, without Renormalization Pruning, we do not get convergence, in general. 
We assume the $\phi$ is random. However, in practice we do train them for a couple epochs, specifically 20 epochs. 
Let pruning set $P$ contain pruned indices of $\Phi_M$. 

\begin{theorem} \label{thm:2}
Assume $\{\phi_i\}_i$ are sampled uniformly, i.i.d., from the $\delta$-sphere in dimension $D$ ($a_i$ can absorb the scalar). Let $a_i > 0$ by absorbing sign in $\phi_i$. Set $\epsilon$ such that $0 < \epsilon < 1$. For dimension $D$ large enough, with probability at least $1-C e^{-c D \epsilon^2/4}$, we have $|\langle \phi_i, \phi_j \rangle|<\epsilon$ for $i \neq j$ and constants $c$ and $C$, see concentrations on the sphere in \cite{ball1997elementary, matousek2013lectures, becker2016new}. 

In this case,

$$ \| \Phi - \Phi_M \|_2^2 
 \ge \delta^2 \sum_{i \in P} a_i^2 
%  - C\epsilon. $$
- \epsilon M^2 \max_{i \in P} a_i^2. $$

\end{theorem}

Proof in appendix. 

\section{Experimental Results}
We test the Renormalized Pruning method, Algo. \ref{algo1}, against the standard pruning method. We compare on three real world datasets: MNIST \cite{mnist}, Fashion MNIST \cite{fasion_mnist}, and CIFAR-10 \cite{cifar10}. Our neural network architecture is a fully connected 3 layer neural network with most nodes in the first hidden layer. We use stochastic gradient descent with momentum to minimize cross entropy loss over 20 epochs. We sparsify the first hidden layer after training. In Figures \ref{fig:mnist}, \ref{fig:fmnist}, and \ref{fig:cifar}, we plot the training and testing accuracy after pruning the trained network (no re-training). The MNIST and Fashion MNIST neural networks have fully connected layers width 6000 then Relu then 30 then Relu then 10, the output layer; 
Input $\rightarrow 6,000 \rightarrow Relu \rightarrow 30 \rightarrow Relu \rightarrow 10= $ Output. 
The CIFAR 10 neural network has fully connected layers width 6000 then Relu then 300 then Relu then 10, the output layer; 
Input $\rightarrow 60,000 \rightarrow Relu \rightarrow 30 \rightarrow Relu \rightarrow 10= $ Output. 

\begin{figure}[h]
    \centering
    \includegraphics[width=9cm]{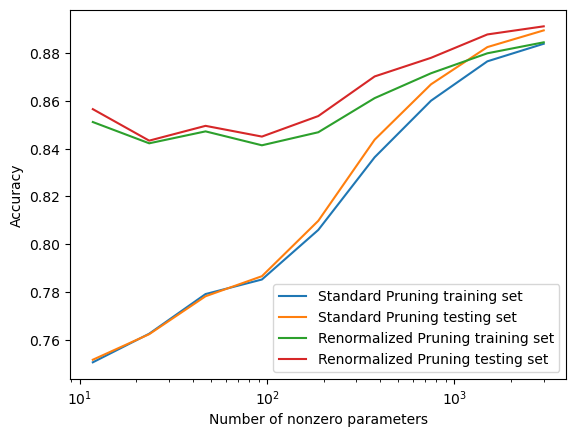}
    \caption{MNIST trained neural network. Plot of accuracy of neural network after sparsifying with versus without renormalization.}
    \label{fig:mnist}
\end{figure}

\begin{figure}[h]
    \centering
    \includegraphics[width=9cm]{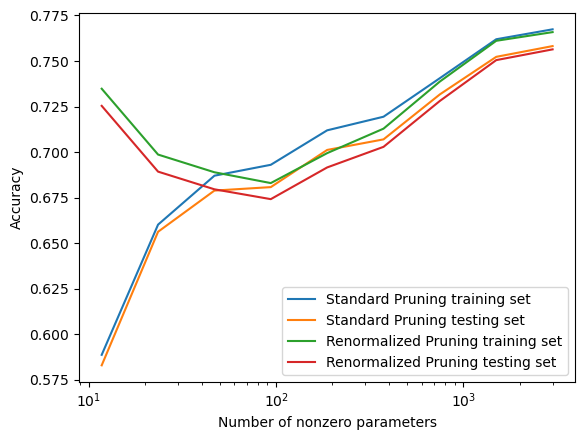}
    \caption{Fashion MNIST trained neural network. Plot of accuracy of neural network after sparsifying with versus without renormalization.}
    \label{fig:fmnist}
\end{figure}

\begin{figure}[h]
    \centering
    \includegraphics[width=9cm]{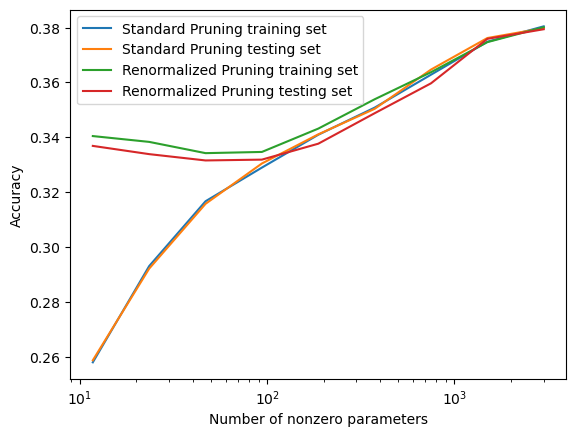}
    \caption{CIFAR-10 trained neural network. Plot of accuracy of neural network after sparsifying with versus without renormalization.}
    \label{fig:cifar}
\end{figure}

Note the logarithmic scales in the figures. We see that standard pruning decreases in accuracy quickly as sparsity increases, that is the number of nonzero parameters goes to zero. We also see that Renormalized Pruning maintains it's accuracy much better in the high sparsity regime. We have not tested this method on the largest neural networks or the largest datasets due to computation constraints but we expect the same result due to the theorems above. 

\section{Conclusion}

We have proposed the Renormalized Pruning method which is simple and fast to implement with the most common `one-shot' pruning method. We proved that Renormalized Pruning has error that goes to 0 with concentration of neural network parameters and feature embeddings. Then we proved that standard pruning (via absolute value) has large error that doesn't converge, with high probability. We, experimentally, see that the standard pruning accuracy decays in the very high sparsity regime, >90\% sparsity. We experimentally test this method on three real world datasets: MNIST, Fashion MNIST, and CIFAR-10. We see that Renormalized Pruning drastically improves accuracy for high pruning levels. We believe that this is an important improvement in sparse neural network theory and that it is simple enough for anyone to immediately use.

% \begin{table}
%   \caption{Sample table title}
%   \label{sample-table}
%   \centering
%   \begin{tabular}{lll}
%     \toprule
%     \multicolumn{2}{c}{Part}                   \\
%     \cmidrule(r){1-2}
%     Name     & Description     & Size ($\mu$m) \\
%     \midrule
%     Dendrite & Input terminal  & $\sim$100     \\
%     Axon     & Output terminal & $\sim$10      \\
%     Soma     & Cell body       & up to $10^6$  \\
%     \bottomrule
%   \end{tabular}
% \end{table}

\bibliography{refs}

%%%%%%%%%%%%%%%%%%%%%%%%%%%%%%%%%%%%%%%%%%%%%%%%%%%%%%%%%%%%

\appendix

\section{Appendix}

\begin{proof}[Proof of Theorem \ref{thm:1}]
Recall we let the renormalized $\widehat \Phi_M := \frac{N}{N-M} \Phi_M$. 
Assume all $\phi_i$ are in the $\delta$ radius 2-Ball, $\phi_i \in B_{\delta}$, that is $\|\phi_i\|_{2} \le \delta$.
Let $a_i \in [\alpha, \beta]$ with $0<\alpha<\beta$ for each $i$ by absorbing the sign into $\phi_i$.
Let $P$ be the set of indices pruned in $\Phi_M$. 

Set $\gamma=M/(N-M)$.

$$  \| \Phi - \widehat\Phi_M \|_2 
 = \| \Phi - (1+\gamma)\Phi_M \| $$
$$ =  \left\| \sum_i a_i \phi_i - (1+\gamma) \sum_{j \notin P} a_j \phi_j \right\| $$ 
$$ =  \left\| \sum_{i \in P} a_i \phi_i - \gamma \sum_{j \notin P} a_j \phi_j \right\| $$
$$ =  \left\| \sum_{i \in P} \left( a_i \phi_i - \frac{\gamma}{M} \sum_{j \notin P} a_j \phi_j \right) \right\| $$

$$ \le \sum_{i \in P} \left\| a_i \phi_i - \frac{\gamma}{M} \sum_{j \notin P} a_j \phi_j \right\| $$

$$ \le  \sum_{i \in P}  \left\| a_i \phi_i - \frac{\gamma}{M} \sum_{j \notin P} a_j \phi_j - a_i \mu_\phi + a_i \mu_\phi \right\|   $$

$$ \le  \sum_{i \in P} 
\left\| a_i \phi_i - a_i \mu_\phi \right\| 
+ \left\| a_i \mu_\phi - \frac{\gamma}{M} \sum_{j \notin P} a_j \phi_j \right\| $$

$$ \le \sum_{i \in P} 
a_i \Delta_{\phi}
+ \left\|  \sum_{j \notin P} \frac{\gamma}{M} a_j \phi_j - \frac{a_i}{(N-M)} \mu_\phi \right\| $$

$$ \le \sum_{i \in P} 
a_i \Delta_{\phi}
+ \left\| \frac{a_i}{N-M} \sum_{j \notin P} \frac{a_j}{a_i} \phi_j - \mu_\phi \right\| $$

$$ \le \sum_{i \in P} 
a_i \Delta_{\phi}
+  \frac{a_i}{N-M} \sum_{j \notin P} \left\| \frac{a_j}{a_i} \phi_j - \mu_\phi \right\| $$

$$ \le \sum_{i \in P} 
a_i \Delta_{\phi}
+  \frac{a_i}{N-M} \sum_{j \notin P} \left\| \phi_j - \frac{a_j}{a_i} \phi_j \right\|
+ \left\| \phi_j - \mu_\phi \right\| $$

$$ \le \sum_{i \in P} 
a_i \Delta_{\phi}
+ \frac{a_i}{N-M} \sum_{j \notin P} \left\| 1 - \frac{a_j}{a_i}  \right\| \|\phi_j\|
+ \Delta_{\phi} $$

$$ \le \sum_{i \in P} 
a_i \Delta_{\phi}
+ \frac{a_i}{N-M} \sum_{j \notin P} \left\| 1 - \frac{a_j}{a_i}  \right\| \delta
+ \Delta_{\phi} $$

$$ \le \sum_{i \in P} 
a_i \Delta_{\phi}
+ \frac{a_i}{N-M} \sum_{j \notin P} \frac{(\beta-\alpha)(\beta+\alpha)}{\alpha \beta} \delta
+ \Delta_{\phi} $$

$$ \le \sum_{i \in P} 
a_i \Delta_{\phi}
+ a_i  \frac{(\beta-\alpha)(\beta+\alpha)}{\alpha \beta} \delta
+ \Delta_{\phi} a_i$$

$$ \le 
 2 \beta \Delta_\phi M
+ (\beta-\alpha)(\beta+\alpha)
\delta M/\alpha
$$

$$ \le
 O(\Delta_\phi+\xi)
$$
as $\Delta_\phi + \xi \rightarrow 0$.

\end{proof}

\begin{proof}[Proof of Theorem \ref{thm:2}]
We have that 
$$ \| \Phi - \Phi_M \|_2^2 
= \| \sum_{i \in P} a_i \phi_i \|_2^2 $$
$$ = \sum_{i \in P} a_i^2 \langle \phi_i, \phi_i \rangle 
+  \sum_{i,j \in P, i \neq j} 2 a_i a_j \langle \phi_i, \phi_j \rangle $$
$$ = \delta^2 \sum_{i \in P} a_i^2
+ \sum_{i,j \in P, i \neq j} 2 a_i a_j  \langle \phi_i, \phi_j \rangle $$
$$ \ge \delta^2 \sum_{i \in P} a_i^2
- \sum_{i,j \in P, i \neq j} 2 a_i a_j \epsilon $$
$$ \ge \delta^2 \sum_{i \in P} a_i^2
- \epsilon M^2 \max_{i,j \in P, i \neq j} a_i a_j $$
$$ \ge \delta^2 \sum_{i \in P} a_i^2
- \epsilon M^2 \max_{i \in P} a_i^2. $$
\end{proof}

\end{document}